\pdfoutput=1
\documentclass[journal]{IEEEtran}
\usepackage{amsmath}
\usepackage{amssymb}
\usepackage{booktabs}
\usepackage{graphicx}
\usepackage{siunitx}
\usepackage{url}
\usepackage{newtxtext}

\newtheorem{proposition}{Proposition}

\begin{document}
\bstctlcite{IEEEexample:BSTcontrol}

\title{Bayesian Distribution System State Estimation with Flow Matching}

\author{Nan Lin,~\IEEEmembership{Graduate Student Member,~IEEE}, Peter~Palensky,~\IEEEmembership{Senior Member,~IEEE}, Pedro~P.~Vergara,~\IEEEmembership{Senior Member,~IEEE}
\thanks{Nan Lin is funded by NWO Align4Energy Project NWA.1389.20.251. This work used the Dutch national e-infrastructure with the support of the SURF Cooperative using grant no. EINF-17304.}%
\thanks{Nan Lin, Peter Palensky, and Pedro P. Vergara are with the Intelligent Electrical Power Grids Group, Delft University of Technology, 2628 CD Delft, The Netherlands (e-mail: N.Lin@tudelft.nl, P.Palensky@tudelft.nl, P.P.VergaraBarrios@tudelft.nl).}%
\thanks{The code is available on Github: https://github.com/sentient-codebot/PoSE and https://github.com/distributionnetworksTUDelft/PoSE}%
}

\maketitle

\begin{abstract}
Medium-voltage feeders have few real-time measurements, so distribution
system state estimation (DSSE) rests on pseudo-measurements of the secondary
substations, a prior in Bayesian terms. Weighted least squares keeps the power
flow exact but reduces this prior to a mean and variance, while supervised
deep learning estimators also learn the power flow and are tied to one
measurement set. We propose PoSE (posterior sampling for state estimation),
which keeps the Bayesian structure of DSSE and learns only what the
distribution system operator does not know, namely the distribution of substation injections, modelled by a neural network trained with flow
matching on historical data. The power flow equations, grid topology and
measurements enter only at sampling, through a noise-aware Kalman-type
correction. The neural network therefore never approximates the power flow,
but every posterior sample is a physically consistent operating point, and
one trained prior serves any measurement set, noise level and topology. On
the IEEE 33-bus feeder, PoSE
lowers the probabilistic injection error (CRPS) by 6.7--7.5\% against a
full-covariance Gaussian prior in all 12 tested
settings. It provides reliable uncertainty: the 90\% intervals contain the
true injection almost exactly as often as exact posterior samples would,
whereas those of supervised estimators contain it in only 37--78\% of cases. 
On the 240-bus unbalanced three-phase Iowa
feeder, it gives the lowest injection, voltage and branch-current errors of
all estimators. 
A fast variant matches this accuracy at 0.2--0.6~s
per snapshot and, without retraining, beats Gaussian priors and
retrained supervised estimators on seven reconfigured topologies.
\end{abstract}


\section{Introduction}
\label{sec:introduction}

Distribution system operators (DSOs) must know the voltages and branch loading of
their medium-voltage feeders to keep them within limits, and rooftop PV pushes
these feeders towards those limits in both directions: voltage rise and
reverse flow at midday, voltage drop and peak loading in the evening. Yet a
typical feeder is measured in real time only at the primary substation and at
a few remote points. Distribution system state estimation (DSSE) is therefore underdetermined, and pseudo-measurements
of the secondary substations, estimates of their injections from historical
data, carry most of the information about the state
\cite{abur2004power,dehghanpour2019survey,primadianto2017review}.

Two requirements follow. The first is a realistic joint description of the
substation injections, which are correlated across substations through shared
weather and consumption patterns and non-Gaussian under rooftop PV. The second
is physical and measurement consistency. The estimator must incorporate
sparse, noisy measurements through the nonlinear power flow equations and
provide calibrated uncertainty.

Existing approaches meet only one requirement at a time. Weighted least squares
with Gaussian pseudo-measurements meets the second. It accepts any measurement
set, weights each measurement by its noise variance and solves the AC power
flow. It describes the injections, however, by independent means and
variances, and both simplifications distort the uncertainty it reports.
Rooftop PV output rises and falls together across substations under shared
irradiance, so an independent prior understates how far the substations can
deviate together, which is what drives midday voltage rise and reverse flow.
When the pseudo-measurements are skewed or heavy-tailed, the WLS
overstates the uncertainty, by a margin that varies across buses and
operating points~\cite{mamudi2026sensitivity}. The WLS covariance therefore does not give a
reliable limit violation probability. Gaussian mixture
pseudo-measurements add a non-Gaussian shape~
\cite{singh2010statistical,singh2010dsse}, but the estimator still returns one
point and one covariance, so a skewed or bimodal injection is reported as a
symmetric interval around a single estimate.

Learned estimators take the statistics of the injections and voltages from
data rather than from an assumed Gaussian shape, but give up parts of the
second requirement. Learned pseudo-measurements improve the point estimate~
\cite{delavarga2025learning,delavarga2026endtoend}. Bayesian deep estimators
approximate the minimum mean square error estimator~
\cite{mestav2019bayesian,mestav2019gan} or the voltage distribution~
\cite{zhang2024multifidelity}, and a conditional diffusion model trained on
measurement--state pairs gives voltage point estimates~
\cite{hua2026spatiotemporal}. Each is trained for one measurement
configuration, so a failed meter, a new sensor, a change in sensor accuracy
or a switching action requires retraining. The estimators that output
voltages see the measurement noise only through simulated training pairs and
must approximate the power flow equations themselves, so their voltages and
branch currents need not form a physically consistent operating point. They
return a point, or a mean and variance of the voltage magnitudes, which does
not give the probability that a branch exceeds its thermal limit. A
generative adversarial prior over unobservable injections, fitted to the
measurements at inference through LinDistFlow, avoids retraining~
\cite{kamal2022physics}. Without a noise model, however, it treats every
measurement as exact, so the measurement noise passes into the estimate, and
its single point estimate satisfies only the linearized power flow. A
normalizing flow has been used to learn the state distribution~
\cite{afrasiabi2024normalizing}. Table~\ref{tab:prior_work} summarizes these
approaches by the quantity each models, either as a regression $y\to\cdot$
from the measurements $y$ or as a distribution $p(\cdot)$, over the injections
$s$ or the voltage phasors $V$, possibly given calendar context $c$. It also
marks whether a new measurement set or noise level needs retraining, and
whether the measurement noise covariance enters the estimator explicitly, only
through simulated training pairs, or not at all.

The design principle we follow separates what the DSO knows from what
must be learned. The power flow equations, the feeder topology, and the sensor
locations and noise are known, and they define the likelihood of the
measurements. The joint distribution of the injections is unknown and is
learned from historical data as the prior. Bayes' rule combines the two into
the posterior, which is available in closed form only for a Gaussian prior and
linear measurements. Diffusion and flow priors have made this posterior
tractable for non-Gaussian priors and nonlinear measurements
\cite{chung2023dps,song2023pigdm,boys2024tmpd,pokle2024trainingfree}, and
SmartMeterFM applies measurement-guided flow matching to smart meter data
\cite{lin2026smartmeterfm}. We train the prior once, and the measurements
enter only at sampling, so the same prior serves any measurement set, noise
level and topology. DSSE adds a combination these methods have not addressed:
noisy measurements of a few aggregate quantities through the AC power flow,
and estimates that must be physically consistent and calibrated.

Our solution, PoSE (posterior sampling for state estimation), is a
conditional flow matching prior \cite{lipman2023flow} over the substation
injections, sampled with a noise-aware endpoint correction linearized through
the power flow model. The contributions are:
\begin{itemize}
\item To our knowledge, PoSE is the first DSSE method that samples the
posterior with a learned generative prior and an inference-time endpoint
correction toward the measurements. The prior over the secondary-substation injections is trained
once on historical data and contains no sensor or noise information. The
measurement set, its noise and the feeder model enter only through the
correction. The same trained prior therefore serves any meter configuration, noise
level and switching configuration of the feeder without retraining, and they
return the posterior distribution of the state rather than a single estimate.
\item PoSE returns estimates that satisfy the AC power flow by
construction. The prior is learned in the injection space, so the network
never has to approximate the power flow equations, as estimators that
output voltages must. Every posterior sample is instead mapped through the
AC power flow, so every sampled voltage and branch current is a physically
consistent operating point. Operational quantities, such as the probability
that a voltage or thermal limit is violated, are read directly from the
samples.
\item We design a noise-aware endpoint correction mechanism, with proven exactness in linear-Gaussian. It uses the full covariance of
the learned prior instead of an isotropic or diagonal approximation. Its fast variant, PoSE-fast, forms this covariance
only along the measured directions, so its cost grows with the number of
measurements, which is small in distribution grids, rather than with the
state dimension.
\end{itemize}

\begin{table*}[!t]
\centering
\caption{Comparison with prior DSSE approaches}
\label{tab:prior_work}
\footnotesize
\setlength{\tabcolsep}{4pt}
\begin{tabular}{@{}lllccccc@{}}
\toprule
Approach & Model family & Modelled quantity &
\begin{tabular}[c]{@{}c@{}}Deep\\learning\end{tabular} & Output &
\begin{tabular}[c]{@{}c@{}}No\\retraining\end{tabular} & Noise &
\begin{tabular}[c]{@{}c@{}}AC-\\consistent\end{tabular} \\
\midrule
WLS, Gaussian pseudo-meas.\ \cite{abur2004power} & Diagonal Gaussian &
Injections, $p(s)$ & $\times$ & Point, covariance & $\checkmark$ & Explicit &
$\checkmark$ \\
GMM pseudo-meas.\ \cite{singh2010dsse} & Gaussian mixture &
Injections, $p(s)$ & $\times$ & Point, covariance & $\checkmark$ & Explicit &
$\checkmark$ \\
Learned pseudo-meas.\ \cite{delavarga2025learning} & Linear regression &
Injections, $y\to s$ & $\times$ & Point, covariance & $\times$ & Explicit &
$\checkmark$ \\
End-to-end pseudo-meas.\ \cite{delavarga2026endtoend} & MLP &
Injections, $y\to s$ & $\checkmark$ & Point, covariance & $\times$ & Explicit &
$\checkmark$ \\
Bayesian deep estimator \cite{mestav2019bayesian} & MLP &
Voltages, $y\to V$ & $\checkmark$ & Point & $\times$ & Training only &
$\times$ \\
GAN-based estimator \cite{mestav2019gan} & GAN; MLP &
Inj., $p(s)$; volt., $y\to V$ & $\checkmark$ & Point & $\times$ &
Training only & $\times$ \\
Multi-fidelity deep GP \cite{zhang2024multifidelity} & Deep Gaussian process &
Voltage magn., $y\to\lvert V\rvert$ & $\checkmark$ & Mean, variance &
$\times$ & Training only & $\times$ \\
Conditional diffusion \cite{hua2026spatiotemporal} & Conditional diffusion &
Voltages, $p(V\mid y)$ & $\checkmark$ & Point & $\times$ & Training only &
$\times$ \\
Physics-conditioned GAN \cite{kamal2022physics} & Wasserstein GAN &
Unobs.\ injections, $p(s)$ & $\checkmark$ & Point & $\checkmark$ & None &
$\times$ \\
\midrule
PoSE (ours) & Flow matching & Injections, $p(s\mid c)$ & $\checkmark$ &
Posterior samples & $\checkmark$ & Explicit & $\checkmark$ \\
\bottomrule
\end{tabular}
\end{table*}

\section{DSSE as a Bayesian Inverse Problem}
\label{sec:problem}

\subsection{State, Measurements and Posterior}
We consider a radial medium-voltage distribution feeder, modeled as a
single-phase equivalent or as an unbalanced three-phase network, with a known
slack voltage at the primary substation and $N$ load nodes, the secondary
substations. The
unknown is the vector $s=[p^{\mathsf{T}}, q^{\mathsf{T}}]^{\mathsf{T}}\in\mathbb{R}^{2N}$ of net active power $p$, in
kW, and net reactive power $q$, in kvar, drawn at the load nodes in one
snapshot, positive for consumption and negative for net export; $s$
determines the operating point. Buses without load are part of the feeder
model, not of the state. An AC power flow maps $s$ to the node voltages and the
branch currents; we write
\begin{equation}
  y=h(s)+e,\qquad e\sim\mathcal{N}(0,R),
  \label{eq:measurement}
\end{equation}
where $h$ runs the power flow and returns the measured quantities: active and
reactive power at the feeder head, voltage magnitudes at selected buses and,
where available, the injections of monitored secondary substations. With few
real-time measurements, $m\ll 2N$, and $s$ is not identifiable from $y$ alone.

Given calendar context $c$ (month, time of day and day type), DSSE is the
inverse problem of recovering $s$ from $y$. Its Bayesian solution is the
posterior
\begin{equation}
  p(s\mid y,c)\;\propto\;p(y\mid s)\,p(s\mid c),
  \label{eq:posterior}
\end{equation}
whose likelihood follows from \eqref{eq:measurement} and whose prior
$p(s\mid c)$ describes how the injections of all substations behave jointly in
that context. Every quantity a DSO monitors is a
function of $s$: the posterior over $s$, pushed through the AC power flow,
gives the posterior over voltages and branch currents, and with it the
probability that a limit is violated.

\subsection{Correspondence with Bayesian Inverse Problems}
Table~\ref{tab:terms} maps the terms of DSSE onto those of Bayesian inverse
problems. Pseudo-measurements act directly as a prior distribution: the
classical weighted least-squares (WLS) estimator with Gaussian
pseudo-measurements \cite{abur2004power} is formally the maximum a
posteriori (MAP) estimate under a diagonal Gaussian prior. Similarly, feeder
unobservability represents ill-posedness, where real-time measurements
alone leave a subspace of injections undetermined, which the prior must
resolve. The estimation quality of an unobservable feeder is therefore
fundamentally governed by prior expressiveness, making a richer joint prior the
direct route to superior state estimates.

\begin{table}[!t]
\centering
\caption{DSSE terms and their Bayesian counterparts}
\label{tab:terms}
\begin{tabular}{@{}ll@{}}
\toprule
DSSE & Bayesian inverse problem \\
\midrule
Pseudo-measurements & Prior $p(s\mid c)$ \\
Real-time measurements and their accuracy & Likelihood $p(y\mid s)$ \\
Unobservable feeder & Ill-posed inverse problem \\
WLS with pseudo-measurements & MAP under a Gaussian prior \\
Estimate with uncertainty & Posterior $p(s\mid y,c)$ \\
\bottomrule
\end{tabular}
\end{table}

\section{PoSE Methodology}
\label{sec:method}

The method keeps the two factors of \eqref{eq:posterior} apart. A flow
matching model learns the prior $p(s\mid c)$ once from historical secondary
substation data; measurements enter only at inference, through a correction of
the sampler's endpoint estimate. The same trained weights therefore serve any
measurement set, noise level and radial topology. Fig.~\ref{fig:method} shows
which part is trained, when, and with which data.

\begin{figure}[!t]
\centering
\includegraphics[width=0.9\columnwidth]{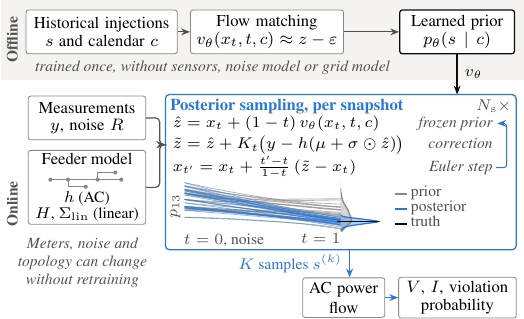}
\vspace{-0.30cm}
\caption{PoSE in two phases. Offline, a flow matching prior over the
substation injections is trained once on historical data, without sensors,
noise model or grid model. Online, the measurements and the feeder model enter
only through the endpoint correction \eqref{eq:correction} at each Euler step
\eqref{eq:euler}. The inset shows 25 real sampling paths of the bus-13
injection for one test snapshot, under the prior and under S2, drawn with an
earlier prior over active power alone.}
\label{fig:method}
\end{figure}

\subsection{Flow Matching Prior over Injections}
We standardize the injections per node, $z=(s-\mu)/\sigma$, with the mean
$\mu$ and standard deviation $\sigma$ of the training snapshots, and train a
conditional flow matching model \cite{lipman2023flow} on the linear path
$x_t=t\,z+(1-t)\,\varepsilon$, $\varepsilon\sim\mathcal{N}(0,I)$, with $t=1$
at the data. A neural network $v_\theta(x_t,t,c)$ regresses $z-\varepsilon$.
Its endpoint estimate
\begin{equation}
  \hat z(x_t,t)=x_t+(1-t)\,v_\theta(x_t,t,c)
  \label{eq:endpoint}
\end{equation}
approximates $\mathbb{E}[z\mid x_t,c]$. The neural network is a residual multilayer perceptron
with adaptive layer normalization on an embedding of $t$ and $c$. Month and
time of day enter through low-order Fourier features of the yearly and daily
cycles, and day type through a learned embedding. Smooth calendar features
let the neural network share information between neighbouring hours and months
instead of fitting each calendar cell separately, which matters because a year
of data offers only a few days per cell. The neural network has four residual blocks
of width 512 (256 for the Fluvius data) and a 128-dimensional embedding.
We train for 10,000 steps with AdamW (learning rate $10^{-3}$, 1,000 warm-up
steps followed by a cosine decay, batch size 1,024, weight decay 0.1),
dropout 0.1 and an exponential moving average of the weights with decay
0.999, and keep the average with the lowest flow matching loss on validation
weeks, evaluated every 2,000 steps. The configuration was selected on
validation data only.

\subsection{Posterior Sampling through Noise-Aware Endpoint Correction}
\label{sec:correction}
At every sampling step the endpoint estimate is corrected towards the
measurements by a Kalman-type update in standardized space,
\begin{align}
  \tilde z&=\hat z+K_t\bigl(y-h(\mu+\sigma\odot\hat z)\bigr),\label{eq:correction}\\
  K_t&=\Sigma_t H_z^{\mathsf T}\bigl(H_z\Sigma_t H_z^{\mathsf T}+R_{\mathrm{eff}}\bigr)^{-1},
  \qquad H_z=H\,\mathrm{diag}(\sigma),\nonumber
\end{align}
where the residual always uses the AC model $h$ and $H$ is the fixed matrix
of LinDistFlow, the linearized DistFlow model
\cite{baran1989sizing}. Its linearization error is folded into the
noise: $R_{\mathrm{eff}}=R+\Sigma_{\mathrm{lin}}$, with
$\Sigma_{\mathrm{lin}}$ the covariance of the difference between the AC and
the LinDistFlow measurements over the training snapshots.

\emph{Unbalanced three-phase networks:} For multi-phase feeders, the state
$s$ remains the $2N$ nodal injections at the secondary substations. The active
and reactive power of each bus is split over its load connections in
proportion to their nominal active power, and the share of a three-phase load
is split equally over its phases. Each connection lies on the phases, or phase
pairs for delta connections, of its distribution transformer winding. The
per-phase injections are thus a fixed linear function of $s$, and the
imbalance adds no unknowns. Voltages and branch currents are indexed over
node-phases and branch-phases. Node-phase $n$ has phase
$\varphi(n)\in\{\mathrm{a},\mathrm{b},\mathrm{c}\}$ with rotation
$\alpha_{\varphi(n)}\in\{1,a,a^2\}$, $a=e^{-j2\pi/3}$. Delta connections are
converted to equivalent wye powers at flat balanced voltages, which gives the
active and reactive injections $P^{\mathrm{eq}}_m$ and $Q^{\mathrm{eq}}_m$ at
node-phase $m$. The linear operator $H$ is the unbalanced LinDistFlow matrix
\cite{gan2014convex}, obtained by linearizing the AC power flow at no load
around the balanced slack voltage $V_s=V_0[1,a,a^2]^{\mathsf T}$, with $V_0$
the slack voltage magnitude. Let $Z^{\mathrm{path}}_{nm}$ be the sum, over the
branches shared by the paths from the slack to node-phases $n$ and $m$, of the
per-unit series impedance entry between phases $\varphi(n)$ and $\varphi(m)$,
and $c_{nm}=\overline{\alpha_{\varphi(n)}}\,\alpha_{\varphi(m)}Z^{\mathrm{path}}_{nm}$.
The first-order expansion of the voltage magnitude is
$|V_n|\approx V_0-\frac{1}{V_0}\sum_m\bigl(\mathrm{Re}(c_{nm})\,P^{\mathrm{eq}}_m+\mathrm{Im}(c_{nm})\,Q^{\mathrm{eq}}_m\bigr)$.
Composing this expansion and the per-phase head flows with the allocation
gives a fixed linear operator $H$ from $s$ to the measurements, so the
correction \eqref{eq:correction} operates under multi-phase physics without
retraining the prior. 

\emph{Endpoint covariance $\Sigma_t$:} The isotropic choice
$r_t^2 I$ with $r_t^2=(1-t)^2/(t^2+(1-t)^2)$ is the pseudoinverse-guided
correction of PiGDM \cite{song2023pigdm} as carried to flows by OT-ODE
\cite{pokle2024trainingfree}. The Tweedie choice uses the conditional
covariance of the endpoint \cite{efron2011tweedie,boys2024tmpd},
\begin{equation}
  \Sigma_t=\mathrm{Cov}(z\mid x_t)=\frac{(1-t)^2}{t}\,
  \frac{\partial\hat z}{\partial x_t},
  \label{eq:tweedie}
\end{equation}
which we symmetrize and project onto the positive semidefinite cone. The full
Jacobian costs one forward-mode product per state dimension and step. The
state dimension of the IEEE 33-bus feeder is 64, so the exact covariance is
affordable. The gain needs
$\Sigma_t$ only through $\Sigma_t H_z^{\mathsf T}$ and
$H_z\Sigma_t H_z^{\mathsf T}$. A fast variant, PoSE-fast, forms only these, as
the cost analysis of TMPD assumes \cite{boys2024tmpd}: $m$ forward- and $m$
reverse-mode products per step, which give the symmetrized product, instead of
$2N$. It projects
the $m\times m$ matrix $H_z\Sigma_t H_z^{\mathsf T}$, whitened by
$R_{\mathrm{eff}}$, onto the positive semidefinite cone; the same directions
are removed from $\Sigma_t H_z^{\mathsf T}$, so the update remains a Kalman
update with a positive semidefinite covariance and never increases the
whitened linear residual. The number of Jacobian products then grows with the
number of measurements rather than with the state dimension, as analyzed in
Section~\ref{sec:cost}.

\emph{Noise:} $R_{\mathrm{eff}}=0$ treats the measurements as exact;
otherwise the noise covariance and the linearization error enter as above.
With $\Sigma_t=r_t^2 I$, $R_{\mathrm{eff}}=0$ recovers the projection guidance
of SmartMeterFM \cite{lin2026smartmeterfm} and $R_{\mathrm{eff}}=R+\Sigma_{\mathrm{lin}}$
the noise-aware isotropic correction of PiGDM and OT-ODE; PoSE pairs the
Tweedie covariance with $R+\Sigma_{\mathrm{lin}}$, as in the TMPD update
\cite{boys2024tmpd}.

The corrected endpoint drives the sampler. We integrate the posterior flow
using the Euler step
\begin{equation}
  x_{t'}=x_t+(t'-t)\,\frac{\tilde z-x_t}{1-t}
  \label{eq:euler}
\end{equation}
with $N_{\mathrm{s}}=200$ steps along the quadratic time grid
$t_i=1-(1-i/N_{\mathrm{s}})^2$, $i=0,\dots,N_{\mathrm{s}}$, which places
more steps near $t=1$, where the narrow conditional posterior is resolved, and
return the corrected endpoint $\tilde z$ of the final step.

\subsection{Exactness for a Gaussian Prior}
\begin{proposition}[Exactness]
\label{prop:exactness}
If the prior on $z$ is Gaussian, $h(s)=c_0+Hs$ is linear (so that
$\Sigma_{\mathrm{lin}}=0$ and $R_{\mathrm{eff}}=R$), $e\sim\mathcal{N}(0,R)$ is
independent of $z$ and of the sampler noise, and
$\hat z=\mathbb{E}[z\mid x_t]$ exactly, then the PoSE correction gives
$\tilde z=\mathbb{E}[z\mid x_t,y]$, and the flow \eqref{eq:euler} in
continuous time transports $\mathcal{N}(0,I)$ to the posterior
$p(z\mid y)$.
\end{proposition}
\begin{IEEEproof}
For a Gaussian prior, $p(z\mid x_t)$ is Gaussian with mean $\hat z$ and
covariance \eqref{eq:tweedie}, and $y$ is linear-Gaussian in $z$ and
independent of $x_t$ given $z$. The Kalman update of this Gaussian by $y$
is \eqref{eq:correction}, so $\tilde z=\mathbb{E}[z\mid x_t,y]$. On the
linear path the velocity of the posterior flow is
$(\mathbb{E}[z\mid x_t,y]-x_t)/(1-t)$, which is \eqref{eq:euler}.
\end{IEEEproof}

The update has the form of the denoising step of TMPD \cite{boys2024tmpd},
adapted here to a linear-path flow sampler. Because conventional WLS with
Gaussian pseudo-measurements corresponds to the MAP estimate under a diagonal
Gaussian prior, PoSE directly reduces to the full Bayesian posterior of WLS
when the prior is Gaussian. In practice, exact posterior recovery depends on
the fidelity of the learned prior, sampler time discretization and power flow
linearization. Because the true distribution of load and rooftop PV injections
is non-Gaussian and unknown in closed form, we verify theoretical exactness on
a linear-Gaussian testbed in Section~\ref{sec:exactness} and evaluate empirical
performance under realistic distribution conditions in
Sections~\ref{sec:case-study} and~\ref{sec:results}.

\subsection{Physically Consistent Samples}
Each posterior sample of $s$ is passed through the AC power flow, which gives
samples of all node voltages and branch currents that satisfy the power flow
equations exactly. Point estimates are posterior means; uncertainty and
violation probabilities follow from the samples without further assumptions.

\section{Case Study}
\label{sec:case-study}

\subsection{Feeder and Power Flow}
The main testbed is the IEEE 33-bus radial feeder \cite{baran1989reconfiguration}
as provided by pandapower (\texttt{case33bw}) \cite{thurner2018pandapower}:
12.66~kV, 32 load nodes and 32 branches, with the slack voltage fixed at
1.0~p.u. A batched backward/forward sweep in PyTorch supplies $h$ and its Jacobian; it
matches pandapower to within $10^{-10}$~p.u.\ in voltage magnitude on
100 test snapshots of each data set. Because
\texttt{case33bw} has no ampacities, the rating of a branch is its current at
nominal loading, and branch currents are reported in percent of that rating.

\emph{Reconfigured feeders:} Branches 1--32 of \texttt{case33bw} feed buses
2--33, and branches 33--37 are the normally open tie switches between buses
8--21, 9--15, 12--22, 18--33 and 25--29. These five ties give seven further
radial topologies, listed in
Table~\ref{tab:topologies}: one that opens branches 7, 9, 14, 32 and 37,
which lowers the losses at nominal loading from 202.68 to 139.55~kW; one with
all five ties closed; and, for each tie, the topology that closes it and
opens the adjacent sectionalizing branch. At nominal loading the largest change of a bus
voltage from the original topology, $\Delta V$, reaches 0.046~p.u. Closing the tie between buses 18 and 33
raises the losses and lowers the minimum voltage; it serves as a stress case.
The voltage meters of S2 stay at buses 18, 22 and 33, and the branch ratings
of each topology are its own currents at nominal loading. The flow matching prior, the
Gaussian and mixture priors and all tuned hyperparameters are those of the
original topology. Only $h$, $H$ and $\Sigma_{\mathrm{lin}}$ are rebuilt. The
supervised estimators are run both as trained on the original topology and
retrained for each topology.

\begin{table}[!t]
\centering
\setlength{\tabcolsep}{3pt}
\caption{Radial topologies of the IEEE 33-bus feeder at nominal loading}
\label{tab:topologies}
\begin{tabular}{@{}llrrr@{}}
\toprule
Topology & Open branches & Loss (kW) & $V_{\min}$ (p.u.) & $\Delta V$ (p.u.) \\
\midrule
Original & 33, 34, 35, 36, 37 & 202.68 & 0.9131 & -- \\
Loss-reducing & 7, 9, 14, 32, 37 & 139.55 & 0.9378 & 0.0412 \\
All ties closed & 7, 9, 14, 28, 32 & 139.98 & 0.9413 & 0.0458 \\
Tie 8--21 & 7, 34, 35, 36, 37 & 158.39 & 0.9299 & 0.0192 \\
Tie 9--15 & 14, 33, 35, 36, 37 & 196.42 & 0.9167 & 0.0134 \\
Tie 12--22 & 11, 33, 34, 36, 37 & 156.79 & 0.9267 & 0.0416 \\
Tie 18--33 (stress) & 32, 33, 34, 35, 37 & 203.95 & 0.9067 & 0.0098 \\
Tie 25--29 & 28, 33, 34, 35, 36 & 175.13 & 0.9285 & 0.0297 \\
\bottomrule
\end{tabular}
\end{table}

\emph{Iowa 240-bus system:} The second feeder is the Iowa distribution test
system \cite{bu2019timeseries}, a real 13.8~kV network with three feeders and
one year of hourly smart meter data from 2017. Net active and reactive powers
are aggregated at primary buses, yielding a 386-dimensional injection state
across 193 load nodes. We evaluate two representations.
First, a balanced positive-sequence equivalent preserving lateral voltage drops,
with two 50~kvar fixed capacitor injections, 1.025~p.u.\ slack voltage and
dead ends removed. On 100 test hours with active regulators, it matches the
unbalanced OpenDSS solution to within 0.004~p.u.\ in bus voltage magnitudes.
Its measurement set is the $S_2$ equivalent, with substation $P$, $Q$ and
three farthest bus voltages.
Second, a full unbalanced three-phase model with 455 node-phases and 432
kept branch-phases, excluding 23 branch-phases carrying no downstream loads or
capacitors, preserving OpenDSS line impedances and lateral phase assignments.
The reference is the reduced OpenDSS model without regulators, transformers or
line charging, which the three-phase power flow reproduces to within
$10^{-11}$~p.u.\ in voltages and $10^{-6}$~kW in per-phase head active power on
100 test hours. The same model generates synthetic measurements and serves the
estimator. Grid model mismatch is out of scope. Across test snapshots, voltage
unbalance remains mild, with unbalance factor at three-phase buses at most
0.13\% and maximum phase spread 0.006~p.u., whereas head current imbalance
reaches 30.6\% with a median of 11.9\%, and 20.8\% of three-phase branch-phase
currents exceed the balanced equivalent by over 10\%. The measurement set is
$S_2\text{-3ph}$ with $m=9$ measurements: per-phase head $P$ and $Q$, and the
voltage magnitudes at bus~1017 phase~c, bus~2048 phase~b and bus~3091 phase~a.

\subsection{Injection Data}
\emph{Ausgrid:} The main testbed uses the Ausgrid solar home electricity data
\cite{ratnam2017residential}: 300 homes with rooftop PV in the distribution network of Ausgrid in New South Wales, Australia,
half-hourly general consumption, controlled load and gross PV generation from
1 July 2010 to 30 June 2013. Each load node aggregates nine homes; the 54
homes with the largest PV systems feed buses 13--18, at the end of the long
branch, and the remaining homes are assigned at random. The net injection of
each node is scaled so that the 99th percentile of its gross consumption over
the training days matches the node's nominal \texttt{case33bw} load, and net
power is scaled by the same factor, so PV export keeps its ratio to
consumption. The reactive injection of a node is its scaled gross consumption
times the tangent of the node's nominal \texttt{case33bw} power factor, with PV
at unity power factor, so reactive and active power differ at midday. A high-PV
scenario multiplies gross generation by three, which
corresponds to about 5~kWp per home and produces reverse flow and voltage rise
at midday.

\emph{Fluvius:} The second data set is one year (2024) of 15-minute
smart meter data from 2,400 Flemish homes published by the DSO
Fluvius \cite{fluvius2024profiles}, aggregated to 75 homes per node
with the PV homes concentrated on buses 13--18. These meters record net
offtake and injection rather than gross consumption, so the reactive injection
is the nominal power factor applied to the net offtake. These data serve the
exactness test described in Section~\ref{sec:exactness} and the study of the training
data volume in Section~\ref{sec:data}.

\emph{Split and context:} Days 8--14 of every month are test days (252 days
for Ausgrid, 84 days each for Fluvius and Iowa), days 22--28 validation days,
and the remaining days training days; the split is strictly by time, never at
random. The Ausgrid data yield 592 training days, the Fluvius data 198, and the
Iowa data 197. From the pool of test days, 1,000 test snapshots are drawn by
stratified sampling round-robin across the 24 hours of the day, ensuring balanced
representation of diurnal operating points across all seasons. The context $c$
consists of month, time of day, and weekday or weekend.

\subsection{Measurement Sets and Noise}

\begin{figure}[!t]
\centering
\includegraphics[width=\columnwidth]{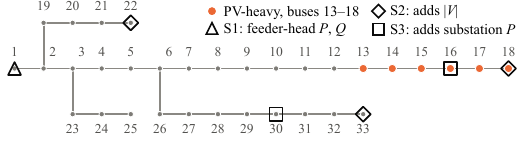}
\vspace{-0.72cm}
\caption{IEEE 33-bus feeder with the PV-heavy secondary substations (buses
13--18) and the measurement sets: S1 measures active and reactive power at the
feeder head, S2 adds voltage magnitudes at buses 18, 22 and 33, and S3 adds the
active power of the substations at buses 16 and 30.}
\label{fig:feeder}
\end{figure}
Three measurement sets cover the observability of a typical medium-voltage
feeder, illustrated in Fig.~\ref{fig:feeder}: S1 holds the active and reactive power at the feeder head; S2 adds the
voltage magnitudes at buses 18, 22 and 33, the ends of the three laterals;
S3 adds the active power of the secondary substations at buses 16 and 30. Each
measurement has Gaussian noise with a standard deviation of one third of a 1\%
maximum error of its nominal magnitude (nominal noise), or three times that
value (high noise). Measurements are generated from the AC power flow of the
true injections.

\subsection{Baselines}
\emph{Alternative estimators:} Six estimators from four families serve as
benchmarks. For classical point estimation, WLS uses diagonal Gaussian
pseudo-measurements with the mean and variance of each node in the training
snapshots of the same month and day type within one hour, solved by
Gauss--Newton on the AC model, with samples from the Laplace approximation at
the estimate. For Bayesian estimation without deep learning, the correlated
Gaussian replaces the diagonal covariance by the Ledoit--Wolf estimate of the
full covariance \cite{ledoit2004wellconditioned} to account for correlations
between substations; the pooled Gaussian retains the full-covariance Gaussian
but pools all training snapshots of the same day type with circular Gaussian
kernels in time of day and month, shrunk towards the diagonal to share
information across calendar contexts; and the Gaussian mixture uses a mixture
prior per context bin \cite{singh2010statistical} with the exact mixture
posterior under the bias-corrected LinDistFlow model. For supervised deep
learning, the supervised MLP maps $(y,c)$ to $s$, retrained for every
measurement set and noise level, with Monte Carlo dropout \cite{gal2016dropout}
for samples. For supervised generative modelling, the supervised flow is a
conditional flow matching model of $p(s\mid c,y)$ with the same architecture and
budget as the prior, retrained for every measurement set and noise level,
following the formulation of supervised conditional diffusion estimators
\cite{hua2026spatiotemporal}.

\emph{Guidance variants with the same prior:} The ablation keeps the learned
prior and changes only how the measurements enter: the projection and
isotropic variants of Section~\ref{sec:correction}, and diffusion posterior
sampling (DPS) \cite{chung2023dps}, the standard gradient guidance of
generative inverse solvers, which moves each sample along the gradient of the
measurement residual of its endpoint estimate.

\emph{Tuning:} Hyperparameters are chosen on the validation weeks: the number
of mixture components and the bandwidths and shrinkage of the pooled Gaussian
(time of day 0.5--4~h, month 0.5--2 months, shrinkage 0--0.2) once per data
set and PV scenario by validation log-likelihood, and the DPS step size from
$10^{-4}$ to 1 separately for every measurement set, noise level and scenario
by validation CRPS. The pooled Gaussian is solved as the correlated Gaussian.
Validation selects a single mixture component on every data set, so the
Gaussian mixture acts as a per-bin Gaussian prior with the exact posterior of
the bias-corrected LinDistFlow model.

WLS is the conventional reference in the main results table. The correlated
Gaussian tests the additional benefit beyond a Gaussian prior that captures
correlations.

\subsection{Metrics}
Every method returns 100 samples per snapshot on the 1,000 stratified test
snapshots, spread evenly across the 24 hours of the day so that every operating
point is equally represented. All methods see the
same truths, measurements and initial noise: each comes from its own random
generator, seeded per snapshot. Injection accuracy is measured
by the continuous ranked probability score (CRPS) of the active power per node
in kW \cite{gneiting2007strictly}; a point estimate scores its absolute error. The
CRPS of the reactive power, in kvar, is reported next to it. Voltage accuracy is the RMSE of the
posterior-mean voltage over the buses, in p.u., and branch accuracy the error
of the posterior-mean current in percent of the rating. Uncertainty calibration is
evaluated by the coverage of the central 90\% sample intervals of the active injections. With $K$
exact samples, such an interval covers a new draw with probability
$0.9(K-1)/(K+1)$, which is 0.882 for $K=100$; we report coverage next to this
attainable value. For operational limit and violation monitoring, we score the detection of limit
violations: voltages outside 0.95--1.05~p.u.\ and branch loading above
100\% of the rating. The predicted violation probability is the fraction of
AC-consistent samples that violate the corresponding limit; we report its Brier score, and the hit
and false-alarm rates of alarms raised at probabilities of 0.5 and 0.1.
Differences between methods are given with 95\% confidence intervals from a
paired block bootstrap over test days.

\section{Results}
\label{sec:results}

\subsection{Exactness at Feeder Scale}
\label{sec:exactness}
We test Proposition~\ref{prop:exactness} at the scale of the feeder. A Gaussian is fitted to the active injections of one midday
hour of the Fluvius data, its analytic endpoint replaces the network, and the exact
posterior for S2 at nominal noise is known in closed form. With 200 steps and
4,000 samples, PoSE matches the posterior: the largest
per-node error of the mean is 1.97 Monte Carlo standard errors, the per-node
standard deviation ratios lie in 0.971--1.017, and the relative Frobenius
error of the covariance is 0.040, below the 0.047 of 4,000 exact draws.

\subsection{Accuracy and Uncertainty Calibration}
\begin{table*}[!t]
\centering
\setlength{\tabcolsep}{4pt}
\caption{Injection accuracy, voltage and branch current errors and coverage, S2 at nominal noise. Bold: best per column.}
\label{tab:main}
\scalebox{0.75}{
\begin{tabular}{@{}l rrrrr rrrrr@{}}
\toprule
 & \multicolumn{5}{c}{Base PV} & \multicolumn{5}{c}{High PV} \\
\cmidrule(lr){2-6}\cmidrule(l){7-11}
Method & CRPS & \shortstack{$\Delta$CRPS vs WLS\\{[95\% CI]}} & V RMSE & Branch & Cov. & CRPS & \shortstack{$\Delta$CRPS vs WLS\\{[95\% CI]}} & V RMSE & Branch & Cov. \\
 & kW & kW & $10^{-3}$~p.u. & \% & & kW & kW & $10^{-3}$~p.u. & \% & \\
\midrule
WLS & 8.27 & Reference & 1.04 & 4.27 & 0.941 & 11.56 & Reference & 1.20 & 5.59 & 0.952 \\
Correlated Gaussian & 6.94 & $-$1.33 [$-$1.47, $-$1.20] & 0.87 & 3.53 & 0.906 & 7.87 & $-$3.68 [$-$4.11, $-$3.27] & 0.94 & 3.92 & 0.904 \\
Pooled Gaussian & 6.76 & $-$1.51 [$-$1.66, $-$1.38] & 0.86 & 3.45 & 0.902 & 7.68 & $-$3.88 [$-$4.30, $-$3.47] & \textbf{0.92} & 3.84 & 0.900 \\
Gaussian mixture & 6.91 & $-$1.36 [$-$1.49, $-$1.22] & 0.93 & 3.54 & \textbf{0.879} & 7.89 & $-$3.67 [$-$4.09, $-$3.27] & 1.14 & 3.98 & \textbf{0.881} \\
Supervised MLP & 7.78 & $-$0.49 [$-$0.66, $-$0.31] & 0.97 & 3.55 & 0.415 & 8.61 & $-$2.95 [$-$3.40, $-$2.53] & 1.10 & 4.04 & 0.490 \\
Supervised flow & 6.87 & $-$1.40 [$-$1.55, $-$1.23] & 0.88 & 3.45 & 0.765 & 7.84 & $-$3.72 [$-$4.15, $-$3.31] & 0.96 & 3.90 & 0.762 \\
PoSE (ours) & \textbf{6.42} & \textbf{\boldmath $-$1.85} [$-$2.00, $-$1.71] & \textbf{0.81} & \textbf{3.28} & 0.874 & \textbf{7.33} & \textbf{\boldmath $-$4.22} [$-$4.66, $-$3.80] & 0.95 & \textbf{3.73} & 0.876 \\
\bottomrule
\end{tabular}%
}
\end{table*}

Table~\ref{tab:main} compares the alternative estimators against WLS in the
setting with voltage measurements at nominal noise. Its $\Delta$CRPS is the
CRPS of a method minus that of WLS, computed before rounding, with the 95\%
paired block-bootstrap confidence interval; negative values favour the method.
For coverage, bold marks the value closest to the attainable 0.882. PoSE lowers the injection
CRPS by 1.85~kW on the base feeder and 4.22~kW on the high-PV feeder relative
to WLS, with both paired confidence intervals below zero. The correlated
Gaussian improves on WLS by accounting for correlations, and pooling calendar
contexts adds 0.18 and 0.20~kW; PoSE lowers the CRPS by a further 0.34 and
0.35~kW against the pooled Gaussian, with the lowest voltage and
branch-current errors on the base feeder and the lowest branch-current error
under high PV.

\begin{figure}[!t]
\centering
\includegraphics[width=\columnwidth]{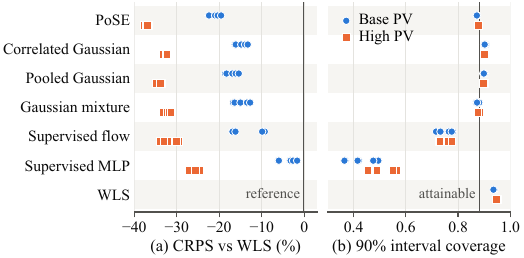}
\vspace{-0.72cm}
\caption{Injection accuracy and calibration of each estimator, one marker per
measurement set and noise level (six per PV scenario), 1,000 test snapshots
each. (a) Injection CRPS relative to WLS, the reference of
Table~\ref{tab:main}. (b) Coverage of the central 90\% intervals of the active
injections; the vertical line is the coverage that exact samples attain with
$K=100$, $0.9(K-1)/(K+1)=0.882$.}
\label{fig:crps}
\end{figure}

Across the whole matrix, shown in Fig.~\ref{fig:crps}(a), three measurement sets, two
noise levels and two PV scenarios, PoSE has a lower CRPS than every
alternative estimator in all 12 settings, with the 95\% confidence interval of
the difference below zero. Against the correlated Gaussian the gain is
7.2--7.5\% on the base feeder and 6.7--7.1\% under high PV, and
0.25--0.43~kW in the midday hours; against the pooled Gaussian it is
4.3--5.0\%. PoSE lowers the branch-current error relative to the correlated
Gaussian in every setting, by 4.3--6.9\%, and has the lowest reactive CRPS of
all alternative estimators in every setting.

Figure~\ref{fig:crps}(b) demonstrates that PoSE produces well-calibrated uncertainty intervals. Its 90\% intervals cover the
truth in 0.870--0.877 of cases across the 12 settings, against an attainable
0.882 for 100 samples, and 0.873--0.881 for the reactive injections. The
Gaussian priors are too wide (correlated 0.900--0.908, pooled 0.896--0.904).
The Gaussian mixture (0.869--0.883) is calibrated as well but less accurate.
The supervised estimators are overconfident: the supervised flow covers
0.716--0.780 and the supervised MLP 0.365--0.563. A prior learned once and
corrected at inference thus gives a posterior that is both accurate and
well calibrated in its uncertainty, while the estimators trained for each measurement set are
overconfident.

The learned posterior is not Gaussian. Across all test cases, the resulting
posteriors are unimodal yet distinctly asymmetric. The absolute skewness of
the PoSE samples at the PV-heavy buses at midday exceeds 0.5 in 24\% of
node-snapshots, against 4\% for the correlated Gaussian, whose samples are
Gaussian by construction so that its share reflects sampling noise alone. The
PoSE-fast posteriors in Fig.~\ref{fig:week}(d) are likewise asymmetric.

PoSE has a lower voltage RMSE than every alternative estimator in 10 of the
12 settings, including all six on the base feeder. Under high PV with voltage measurements
at nominal noise (S2 and S3), the pooled Gaussian reaches 0.92 and
0.87~$\times10^{-3}$~p.u.\ and the correlated Gaussian 0.94 and 0.88, against
0.95 and 0.89 for PoSE. In those two high-PV settings, the normalized chi-squared of the PoSE residual
is 2.0--4.2 across the three measurement sets, reflecting the increased
LinDistFlow linearization error under heavy reverse flow, compared to at most
0.9 under nominal conditions.

\subsection{Operational Limit and Violation Monitoring}
\begin{table*}[!t]
\centering
\setlength{\tabcolsep}{3.5pt}
\caption{Limit monitoring, S2 at nominal noise. Bold: best per column.}
\label{tab:dso}
\scalebox{0.75}{
\begin{tabular}{@{}l rrrrr rrrrr@{}}
\toprule
 & \multicolumn{5}{c}{Base PV} & \multicolumn{5}{c}{High PV} \\
\cmidrule(lr){2-6}\cmidrule(l){7-11}
Method & V hit & V FA & V Brier & I Brier & I err p95 & V hit & V FA & V Brier & I Brier & I err p95 \\
 & \% & \% & $10^{-3}$ & $10^{-3}$ & \% & \% & \% & $10^{-3}$ & $10^{-3}$ & \% \\
\midrule
WLS & 85.0 & \textbf{0.10} & 2.9 & 1.9 & 32.4 & 91.5 & 0.23 & 5.0 & 12.9 & 40.5 \\
Correlated Gaussian & 91.5 & 0.15 & 2.1 & 1.6 & 30.7 & 95.5 & 0.22 & 3.4 & 6.4 & 31.9 \\
Pooled Gaussian & 91.5 & 0.13 & 2.0 & 1.6 & 30.6 & 95.7 & 0.20 & \textbf{3.3} & 6.3 & 31.3 \\
Gaussian mixture & \textbf{96.9} & 0.31 & 2.7 & 1.6 & 30.9 & 91.7 & 0.21 & 5.1 & 7.0 & 31.8 \\
Supervised MLP & 84.5 & 0.16 & 3.3 & 1.7 & 32.4 & 93.0 & 0.25 & 4.9 & 7.1 & 32.2 \\
Supervised flow & 92.2 & 0.17 & 2.4 & 1.6 & 31.3 & \textbf{96.1} & 0.27 & 3.6 & 6.5 & 33.0 \\
PoSE (ours) & 91.8 & \textbf{0.10} & \textbf{1.9} & \textbf{1.5} & \textbf{28.2} & 94.6 & \textbf{0.16} & 3.5 & \textbf{6.2} & \textbf{29.3} \\
\bottomrule
\end{tabular}%
}
\end{table*}

Table~\ref{tab:dso} reports what a DSO monitoring limits would see.
Voltage alarms are raised at a predicted probability of at least 0.5; the hit
rate is the share of true violations that raise an alarm, and the false-alarm
rate the share of non-violating bus-snapshots that do. The last column of each
scenario is the 95th percentile over snapshots of the largest branch-current
error in a snapshot. The base feeder has 613 voltage violations and 62
overloads (bus or branch and snapshot), the high-PV feeder 1,891 and 2,119.
PoSE has the lowest 95th percentile of the largest branch-current error of a
snapshot and the lowest overload Brier score in both scenarios. Against the
correlated Gaussian, this tail error falls from 30.7--33.8\% to 28.2--31.3\%
of the rating across all 12 settings. For voltages, PoSE raises the fewest false alarms in both scenarios (0.10\% and
0.16\%, tied with WLS on the base feeder) while maintaining a hit rate within
1.0 percentage point of the correlated Gaussian. Although the Gaussian mixture
flags more violations on the base feeder, it does so at three times the
false-alarm rate (0.31\%), and Gaussian priors yield comparable voltage Brier
scores (3.3--3.4 against 3.5 under high PV) while suffering from significantly
worse branch-current tail errors.

Fig.~\ref{fig:week} shows the estimates of PoSE-fast and WLS over one test
week of the high-PV feeder. At midday, when PV export dominates, the 90\%
intervals of the bus-13 injection are about a third as wide for PoSE-fast as
for WLS, with a median width of 64 against 215~kW over the 63 midday
half-hours. At the marked snapshot, the 221~kW interval of the learned prior
falls to 75~kW with the feeder-head power alone and changes little with the
further meters of S2 and S3.
Every voltage and branch-current sample is an AC power flow solution, so the
violation probabilities follow from the same samples. Over this week the
Brier score of the overloads of branch 12 is 0.010 for PoSE-fast against
0.015 for WLS. At bus 17, next to the voltage meter at bus 18, both reach
0.021.

\begin{figure*}[!t]
\centering
\includegraphics[width=\textwidth]{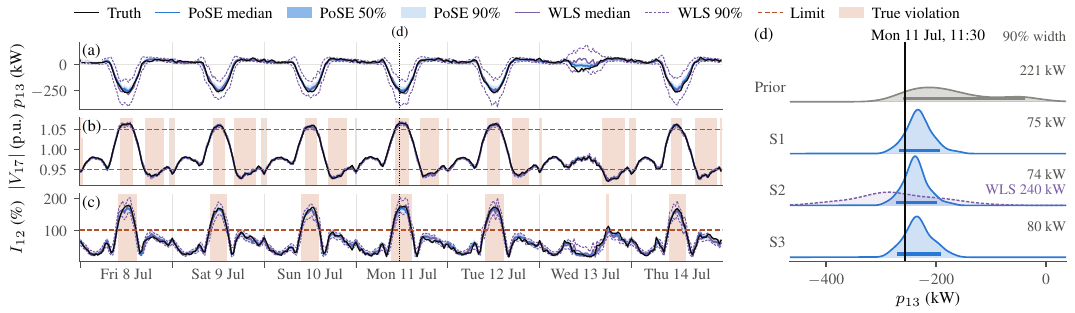}
\vspace{-0.7cm}
\caption{Estimates over the test week of 8--14 July 2011 on the high-PV
feeder, S2 at nominal noise, 100 samples. (a) Net active injection at bus 13.
(b) Voltage magnitude at bus 17, the unmeasured bus with the most true voltage
violations in the week. (c) Current of branch 12, which feeds bus 13, in
percent of its rating. Bands are the central 50\% and 90\% sample intervals
of PoSE-fast, dashed lines those of WLS, and shading marks true violations.
(d) Kernel density estimates of the bus-13 injection at the marked snapshot
under the learned prior and the PoSE-fast posteriors for S1, S2 and S3, with
the WLS posterior for S2 dashed. Bars and numbers give the 90\% intervals.}
\label{fig:week}
\end{figure*}

\subsection{Ablation of Measurement Guidance Mechanisms}
\label{sec:ablation}
\begin{table*}[!t]
\centering
\setlength{\tabcolsep}{4pt}
\caption{Injection CRPS (kW) and 90\% coverage (in parentheses) of the
guidance variants. Bold: lowest CRPS per column.}
\label{tab:ablation}
\scalebox{0.75}{
\begin{tabular}{@{}l rrr rrr@{}}
\toprule
 & \multicolumn{3}{c}{Base PV} & \multicolumn{3}{c}{High PV} \\
\cmidrule(lr){2-4}\cmidrule(l){5-7}
Guidance & S1 & S2 & S3 & S1 & S2 & S3 \\
\midrule
Correlated Gaussian (reference) & 7.07 (0.908) & 6.94 (0.906) & 6.65 (0.900) & 8.06 (0.906) & 7.87 (0.904) & 7.39 (0.901) \\
Projection & 6.82 (0.874) & 14.76 (0.689) & 15.11 (0.632) & 7.90 (0.878) & 17.65 (0.708) & 17.82 (0.657) \\
Isotropic & 6.80 (0.876) & 6.69 (0.874) & 6.36 (0.875) & 7.87 (0.881) & 7.70 (0.878) & 7.24 (0.879) \\
DPS & \textbf{6.54} (0.868) & 6.48 (0.866) & 6.33 (0.841) & \textbf{7.50} (0.871) & 7.42 (0.863) & 7.15 (0.843) \\
PoSE & \textbf{6.54} (0.872) & \textbf{6.42} (0.874) & \textbf{6.15} (0.872) & 7.52 (0.875) & \textbf{7.33} (0.876) & \textbf{6.87} (0.877) \\
\bottomrule
\end{tabular}%
}
\vspace{-0.3cm}
\end{table*}

Table~\ref{tab:ablation} keeps the learned prior and changes only the
guidance, at nominal noise. Accounting for the measurement noise is essential: projection works
when only the feeder-head power is measured, with a CRPS 2.0--4.0\% below the
correlated Gaussian across both noise levels, but with voltage measurements its CRPS is
2.1--6.6 times that of the correlated Gaussian, because exact projection onto
noisy voltages forces the injections to fit the noise. The noise-aware
isotropic correction recovers most of the accuracy, and the Tweedie covariance
supplies the rest: the isotropic correction loses 41--52\% of the gain of PoSE
over the correlated Gaussian on the base feeder and 57--73\% under high PV,
where the conditional covariance of the injections is least isotropic.

DPS, with a step size tuned per setting, comes closest to PoSE in point
accuracy but is overconfident, with coverage 0.836--0.871, and its gap to PoSE
grows as measurements are added: PoSE is 0.06--0.09~kW below DPS in S2 and
0.16--0.28~kW in S3, with 95\% confidence intervals below zero. With only the
feeder-head power measured (S1), the two achieve comparable CRPS, but DPS
requires sensitive manual step-size retuning for every measurement setting.
Furthermore, while DPS has a fast per-step evaluation (0.17~s), it lacks the
exact covariance guidance needed for uncertainty calibration, a computational gap that
PoSE-fast closes without loss of accuracy, as detailed in Section~\ref{sec:cost}.

\subsection{Sensitivity to Historical Training Data Volume}
\label{sec:data}

\begin{figure}[!t]
\centering
\includegraphics[width=\columnwidth]{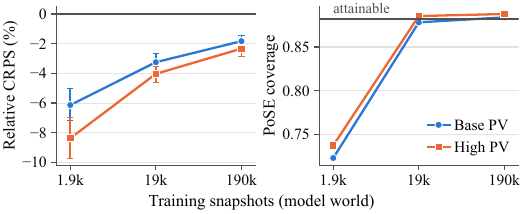}
\vspace{-0.72cm}
\caption{Model world (active-power state): a flow prior trained on all real
days of the Fluvius data serves as the true distribution. Left: CRPS of PoSE
relative to the correlated Gaussian with 95\% paired block-bootstrap
intervals; right: coverage of PoSE, both for S2 at nominal noise and 1,000
test snapshots, as the number of independent training snapshots grows.}
\label{fig:scaling}
\end{figure}
We used a model world to isolate the amount of data: a flow prior
trained on all real Fluvius days serves as the true distribution, and only the
number of independent training snapshots changes, as illustrated in Fig.~\ref{fig:scaling}.
Across $1{,}900$, $19{,}000$ and $190{,}000$ training snapshots, PoSE
consistently outperforms the correlated Gaussian: from $19{,}000$ to $190{,}000$ snapshots, the
CRPS of PoSE falls from 4.17 to 4.13~kW, while that of the correlated Gaussian
decreases from 4.31 to 4.20~kW. Crucially, the coverage of PoSE reaches the attainable
value from $19{,}000$ snapshots on.

\begin{table}[!t]
\centering
\setlength{\tabcolsep}{4pt}
\caption{Injection CRPS (kW) on reconfigured feeders. Bold: lowest on every
topology.}
\label{tab:reconf}
\def\twoline#1#2{\begin{tabular}{@{}l@{}}#1\\#2\end{tabular}}
\scalebox{0.72}{
\begin{tabular}{@{}l rr rr rr rr@{}}
\toprule
 & \multicolumn{4}{c}{Base PV} & \multicolumn{4}{c}{High PV} \\
\cmidrule(lr){2-5}\cmidrule(l){6-9}
 & \multicolumn{2}{c}{S1} & \multicolumn{2}{c}{S2} & \multicolumn{2}{c}{S1} & \multicolumn{2}{c}{S2} \\
\cmidrule(lr){2-3}\cmidrule(lr){4-5}\cmidrule(lr){6-7}\cmidrule(l){8-9}
Method & Orig. & Reconf. & Orig. & Reconf. & Orig. & Reconf. & Orig. & Reconf. \\
\midrule
WLS & 8.45 & 8.45--8.46 & 8.27 & 8.23--8.31 & 12.09 & 12.05--12.09 & 11.56 & 11.48--11.66 \\
Correlated Gaussian & 7.07 & 7.06--7.07 & 6.94 & 6.92--6.99 & 8.06 & 8.06--8.06 & 7.87 & 7.86--7.96 \\
Pooled Gaussian & 6.88 & 6.88--6.88 & 6.76 & 6.73--6.80 & 7.87 & 7.87--7.87 & 7.68 & 7.66--7.77 \\
Gaussian Mixture & 7.04 & 7.03--7.04 & 6.91 & 6.91--6.97 & 8.08 & 8.06--8.08 & 7.89 & 7.88--7.97 \\
\twoline{Supervised MLP,}{retrained} & 7.93 & 7.93--7.93 & 7.78 & 7.78--7.86 & 8.81 & 8.80--8.84 & 8.61 & 8.62--8.73 \\
\twoline{Supervised MLP,}{original topology} & 7.93 & 7.92--7.93 & 7.78 & 7.82--8.38 & 8.81 & 8.79--8.81 & 8.61 & 8.80--9.93 \\
\twoline{Supervised flow,}{retrained} & 7.01 & 7.01--7.02 & 6.87 & 6.89--6.95 & 7.99 & 7.99--8.01 & 7.84 & 7.86--7.96 \\
\twoline{Supervised flow,}{original topology} & 7.01 & 7.01--7.01 & 6.87 & 6.95--7.63 & 7.99 & 7.99--8.01 & 7.84 & 8.03--8.96 \\
\twoline{PoSE-fast}{(ours)} & \textbf{6.54} & \textbf{6.54--6.54} & \textbf{6.42} & \textbf{6.40--6.46} & \textbf{7.52} & \textbf{7.51--7.52} & \textbf{7.33} & \textbf{7.31--7.41} \\
\bottomrule
\end{tabular}
}
\end{table}

\subsection{Scalability on the Unbalanced Three-Phase Iowa Feeder}
\begin{table}[!t]
\centering
\setlength{\tabcolsep}{2.8pt}
\caption{Performance on the unbalanced three-phase 240-bus Iowa feeder ($S_2\text{-3ph}$, nominal noise, 1,000 test snapshots). Bold: best per column; for coverage, closest to the attainable 0.882.}
\label{tab:iowa}
\begin{tabular}{@{}l rrrrrr@{}}
\toprule
Method & CRPS $P$ & CRPS $Q$ & $V$ RMSE & Branch & Cov. & Time \\
 & kW & kvar & $10^{-3}$~p.u. & \% & 90\% & s \\
\midrule
WLS & 1.30 & 0.69 & 0.21 & 5.70 & \textbf{0.881} & \textbf{0.11} \\
Correlated Gaussian & 1.22 & 0.74 & 0.16 & 5.41 & 0.884 & 0.20 \\
Pooled Gaussian & 1.25 & \textbf{0.66} & 0.18 & 5.43 & 0.934 & 0.12 \\
Gaussian mixture & 1.36 & 0.77 & 0.18 & 5.65 & 0.735 & 0.43 \\
PoSE-fast (ours) & \textbf{1.13} & \textbf{0.66} & \textbf{0.15} & \textbf{4.79} & 0.810 & 0.58 \\
\bottomrule
\end{tabular}
\end{table}
On the unbalanced three-phase Iowa feeder, with 455 node-phases, 432
branch-phases and a 386-dimensional injection state, the learned prior applies
directly without retraining: multi-phase physics enters solely through the
unbalanced LinDistFlow operator and AC power flow. As summarized in
Table~\ref{tab:iowa}, PoSE-fast achieves the lowest active injection CRPS at
1.13~kW, reducing by 7.4\% relative to the correlated Gaussian at 1.22~kW,
and outperforming the pooled Gaussian at 1.25~kW, WLS at 1.30~kW, and the
Gaussian mixture at 1.36~kW, with all paired confidence intervals strictly
below zero.
PoSE-fast also attains the lowest overall voltage RMSE of $0.15\times 10^{-3}$~p.u.\
and branch-current error of $4.79$\% of rating against $5.41\text{--}5.70$\%
for the baselines, with 95\% bootstrap intervals below zero across all snapshots;
voltage RMSE ties the correlated Gaussian at midday. On reactive CRPS it ties
the pooled Gaussian at $0.664$~kvar, where the difference is within bootstrap
error. Compared to the single-phase balanced equivalent, where PoSE-fast obtains
1.13~kW CRPS and 0.816 coverage, three-phase results are virtually identical,
differing by only $-0.003$~kW, which confirms that multi-phase physics is handled
entirely at inference time. In contrast, the Gaussian mixture degrades by 11.4\% in CRPS,
from 1.22 to 1.36~kW. As with the balanced feeder, 197 training days leave the
learned prior's coverage at 0.810 against an attainable 0.882. Voltages remain
within 1.011--1.025~p.u., yielding zero voltage violations. For the 1,648
branch-phase overloads, the pooled Gaussian achieves the lowest false-alarm
rate at $3.8\times 10^{-4}$ against $9.4\times 10^{-4}$ for PoSE-fast, while
correlated Gaussian and mixture exhibit higher hit rates at $p \ge 0.5$, reaching
18.8\% and 21.4\% against 16.9\%. 
While exact full-covariance PoSE takes 76~s per
snapshot on the 386-dimensional state on one H100 on the balanced feeder,
PoSE-fast reduces this runtime to 0.58~s in the three-phase network, confirming
that exact measurement guidance scales to multi-phase grids in real time.

\subsection{Reconfigured Feeders without Retraining}
\label{sec:reconf}

Reconfiguration changes the feeder model but not the substation injections,
so the learned prior carries over unchanged. Rebuilding $h$, $H$ and
$\Sigma_{\mathrm{lin}}$ for a new topology takes 0.4--1.8~s, against 52--59~s
to retrain a supervised estimator on one H100. Table~\ref{tab:reconf}
reports PoSE-fast on the seven reconfigured topologies, both PV scenarios and
S1 and S2 at nominal noise, 28 settings. For each estimator it gives the
CRPS on the original topology (Orig.) and its range over the seven
reconfigured ones (Reconf.); on the
original topology, the supervised estimators trained there are the retrained
ones, so the two rows coincide in the Orig.\ columns. PoSE-fast has a lower CRPS than WLS,
the three Gaussian priors and both supervised estimators, retrained or
trained on the original topology, in all 28 settings, with every 95\%
interval below zero. Against the correlated Gaussian the gain is
6.8--7.6\%, against the pooled Gaussian 4.5--5.0\%, as on the original
topology. Its coverage stays at 0.871--0.878, against an attainable 0.882,
and it samples in 0.21--0.23~s per snapshot.

Supervised estimators trained on the original topology carry over when only
the feeder-head power is measured: in S1 their difference to the retrained
model includes zero in 11 (MLP) and 14 (flow) of 14 settings, and the MLP
trained on the original topology is 0.03--0.05~kW better than the retrained
one in three high-PV settings. With voltage measurements they lose accuracy:
in S2 they are worse than retrained in 13 (MLP) and 14 (flow) of 14
settings, by up to 1.22 and 1.00~kW under high PV on the loss-reducing
topology and with all ties closed.

Across all 28 reconfigured settings, PoSE-fast consistently maintains the lowest
reactive CRPS and branch-current error, achieving the lowest 95th
percentile of the largest branch-current error in 27 settings and the lowest
voltage RMSE in 24 settings. It also attains the lowest overload Brier score in
26 settings, matching or exceeding the reliability of Gaussian and mixture
estimators while operating on unseen topologies without retraining.

\subsection{Computational Efficiency and Runtime Analysis}
\label{sec:cost}
On the IEEE 33-bus feeder on one NVIDIA H100, PoSE takes 15.3~s per snapshot
with 200 steps and 100 samples, against 0.03~s for the correlated and pooled
Gaussians. The
difference to the isotropic correction reflects the Tweedie covariance of the
64-dimensional state. Fifty steps give essentially the same CRPS (6.426 vs.\
6.417~kW on base, 7.334 vs.\ 7.332~kW under high PV, S2 nominal) at 3.7--3.9~s
per snapshot, with coverage 0.865 and 0.869. 
PoSE-fast reduces runtime to 0.23~s per snapshot on the base feeder and 0.53~s on
the balanced Iowa network, reaching 0.58~s in the three-phase setting, 67 and about 140
times faster than PoSE while matching its accuracy and calibration. Across all
1,000 test snapshots of five settings (base feeder S2; high PV S1, S2 and S3;
Iowa S2; nominal noise), its injection CRPS differs from PoSE by $-0.009$ to
$0.000$~kW, with every 95\% paired interval within $\pm0.015$~kW; coverage
differs by at most 0.0013 and branch-current error by at most 0.005\% of rating.
Voltage RMSE differs by $-1.1$\% on Iowa to $+1.4$\% under high PV S3.

\section{Conclusion}
\label{sec:conclusion}

PoSE formulates DSSE as posterior sampling and separates what the DSO knows
from what must be learned. A flow matching prior learns the joint distribution
of substation injections from historical data, while power flow equations,
topology and sensors enter only at sampling through a noise-aware endpoint
correction. The
network never approximates the power flow, and every posterior sample, mapped
through the AC power flow, is a physically consistent operating point.

The learned prior provides the realistic representation of injections that
DSSE requires. On the IEEE 33-bus feeder with three years of recorded load and
PV data, PoSE lowers injection CRPS against Gaussian, mixture and supervised
estimators across all evaluated measurement sets, noise levels and PV
penetrations, and on the 240-bus unbalanced three-phase Iowa feeder it achieves
the lowest injection, voltage and branch-current errors. The noise-aware correction
brings in measurements consistently, providing calibrated uncertainty on the
33-bus feeder without retraining. PoSE-fast matches this accuracy at 0.2--0.6~s
per snapshot. Because topology enters only at sampling, the unchanged prior
carries over to seven reconfigured topologies, maintaining its lead over
Gaussian priors and retrained supervised estimators.

Future work follows the same separation. On the prior side, the calibration
of the uncertainty depends on the amount of training data, and composing feeder
priors from per-household models serves DSOs with short histories, and time
series or weather forecasts enter by modifying only the prior network.

\bibliographystyle{IEEEtran}
\bibliography{references}

\end{document}